\documentclass[conference]{IEEEtran}
\usepackage{cite}
\usepackage{amsmath,amssymb,amsfonts}
\usepackage{amsthm}
\usepackage[hidelinks]{hyperref}
\usepackage{graphicx}

\newtheorem{theorem}{Theorem}

\begin{document}

\title{Universal Approximation Theorem for a Single-Layer Transformer}

\author{\IEEEauthorblockN{Esmail Gumaan}
\IEEEauthorblockA{Department of Computer Science, University of Sana'a\\ \href{mailto:23160148@ue.ye.edu}{esmailG231601@ue.ye.edu}}}

\maketitle

\begin{abstract}
Deep learning employs multi-layer neural networks trained via the backpropagation algorithm. This approach has achieved remarkable success across many domains, and has been facilitated by improved training techniques such as adaptive gradient methods (e.g., the Adam optimizer). Sequence modeling has progressed from recurrent neural networks to attention-based models, culminating in the Transformer architecture. Transformers have attained state-of-the-art results in natural language processing (e.g., BERT and GPT-3) and have been applied to computer vision and computational biology. However, theoretical understanding of these models lags behind their empirical success. In this paper, we focus on the mathematical foundations of deep learning and Transformers and present a novel theoretical contribution. We review key concepts from linear algebra, probability, and optimization that underpin deep learning, and we examine the Transformer's multi-head attention mechanism and the backpropagation algorithm in detail. Our primary contribution is a new universal approximation theorem for Transformers: we formally prove that a single-layer Transformer (one self-attention layer with a feed-forward network) is a universal function approximator under suitable conditions. We provide a rigorous statement of this theorem and an in-depth proof. Finally, we discuss case studies that illustrate the implications of this theoretical result in practice. Our findings deepen the theoretical understanding of Transformers and help bridge the gap between deep learning practice and theory.
\end{abstract}

\section{Introduction}
Deep learning employs multi-layer artificial neural networks trained via the backpropagation algorithm \cite{rumelhart1986}. This approach, combined with advances in architectures and optimization techniques, has led to remarkable success across many domains \cite{lecun2015,goodfellow2016}. In particular, sequence modeling has progressed from recurrent network models \cite{sutskever2014} to attention-based encoder-decoder models \cite{bahdanau2015}, culminating in the Transformer architecture \cite{vaswani2017}. Transformers have achieved state-of-the-art results in natural language processing, powering models such as BERT \cite{devlin2019} and GPT-3 \cite{brown2020}, and have also been applied to other fields including image recognition \cite{dosovitskiy2021} and protein structure prediction \cite{jumper2021}. These advances have been facilitated by improved training methods, including adaptive gradient-based optimizers like Adam \cite{kingma2015}.

Despite these empirical successes, a deep theoretical understanding of why such models work so well remains an active area of research. A foundational result in neural network theory is the universal approximation theorem \cite{cybenko1989,hornik1989}, which guarantees that a sufficiently large feed-forward network can approximate any continuous function on a compact domain. Numerous works have extended this result to deeper or more specialized network architectures \cite{barron1993,lu2017,yarotsky2017}. However, analogous universal approximation results for Transformer models have only begun to emerge recently \cite{yun2020,kajitsuka2024}.

In this paper, we bridge this gap by proving a new universal approximation theorem for Transformers. In particular, we show that even a single self-attention layer combined with a simple feed-forward network is a universal function approximator under suitable conditions. To provide context, we first review the mathematical foundations relevant to deep learning (Section~II), covering key concepts from linear algebra, probability, and optimization. We then describe the Transformer's attention mechanism in detail (Section~III) and review the backpropagation algorithm and optimization methods for training deep networks (Section~IV). Section~V presents our main theoretical contribution: the formal statement and proof of the universal approximation theorem for a single-layer Transformer model. In Section~VI, we discuss case studies and practical implications of this theorem, and Section~VII concludes the paper.

\section{Mathematical Foundations}
Deep learning is grounded in several fundamental areas of mathematics, notably linear algebra, probability theory, and optimization. In this section, we provide a brief overview of these foundations and how they relate to deep learning models.

\subsection{Linear Algebra and Neural Networks}
Linear algebra provides the language for describing computations in neural networks. Vectors and matrices are used to represent data, parameters, and transformations. For example, an input to a neural network can be represented as a vector $x \in \mathbb{R}^d$, and the parameters of a fully connected layer can be represented by a weight matrix $W \in \mathbb{R}^{d' \times d}$ and a bias vector $b \in \mathbb{R}^{d'}$. The output of such a layer is given by
\begin{equation}
h = \phi(Wx + b)\,,
\end{equation}
where $\phi(\cdot)$ is a nonlinear activation function (such as ReLU or sigmoid) applied element-wise. This operation is a basic building block of deep networks, and stacking many such layers (with intermediate activations) yields a complex function with learned parameters $W, b$ at each layer.

Many concepts from linear algebra are directly relevant to understanding deep learning. For instance, the notions of linear independence, matrix rank, and eigenvalues can help explain the capacity of layers to transform data. Singular value decomposition (SVD) is used to analyze weight matrices, and inner products (dot products) quantify similarities, which is especially important in understanding attention mechanisms (Section~III).

\subsection{Probability and Learning}
Probability theory underlies many aspects of machine learning, particularly in modeling uncertainties and defining learning objectives. In supervised learning, we often assume a data distribution and aim to minimize the expected loss:
\begin{equation}
\min_{\theta} \ \mathbb{E}_{(x,y)\sim \mathcal{D}}[\ell(f_{\theta}(x),\,y)]\,,
\end{equation}
where $f_{\theta}(x)$ is the model's prediction for input $x$ (with parameters $\theta$), $y$ is the true target, and $\ell(\cdot,\cdot)$ is a loss function. In practice, this expectation is approximated by an average over the training samples.

A common choice of loss function for classification tasks is the cross-entropy loss, which is derived from probabilistic principles. If the model outputs a probability distribution $\hat{y}$ over $C$ classes (typically via a softmax layer), and $y$ is the one-hot encoded true label, the cross-entropy loss is:
\begin{equation}
\mathcal{L}(y,\hat{y}) = -\sum_{c=1}^{C} y_c \, \log \hat{y}_c\,,
\end{equation}
which can be interpreted as the negative log-likelihood of the correct class. Minimizing this loss encourages the model to assign high probability to the correct class.

Deep learning also relies on probabilistic concepts like Bayes' rule (for probabilistic models), Markov chains (for certain generative models), and latent variable models. Regularization techniques (such as dropout) can be viewed through a probabilistic lens as introducing randomness to prevent overfitting.

\subsection{Optimization and Gradient-Based Training}
Training a deep neural network involves solving a high-dimensional optimization problem. The objective is typically to find model parameters $\theta$ that minimize a loss function $L(\theta)$, which measures the model's error on the training data. This is challenging because $L(\theta)$ is usually non-convex and may have many local minima or saddle points.

The predominant approach for training is gradient-based optimization. The gradient $\nabla_{\theta}L$ provides the direction of steepest increase in the loss with respect to the parameters. By moving in the opposite direction of the gradient, one can decrease the loss. The simplest incarnation is \emph{gradient descent}, where the update rule for parameters is:
\begin{equation}
\theta \leftarrow \theta - \eta \, \nabla_{\theta}L(\theta)\,,
\end{equation}
with $\eta > 0$ a chosen learning rate. In practice, \emph{stochastic gradient descent} (SGD) is used: the gradient is estimated on a mini-batch of training examples rather than the entire dataset, which introduces some noise but is much more efficient.

Numerous improvements on basic SGD have been developed to speed up and stabilize training. These include momentum methods, adaptive learning rate methods, and others. For example, the Adam optimizer \cite{kingma2015} adapts the learning rate for each parameter using estimates of first and second moments of the gradients, often resulting in faster convergence. Such optimizers are crucial for training large Transformer models.

Understanding optimization in deep learning also involves concepts from calculus (e.g., the chain rule for computing gradients) and sometimes dynamical systems (viewing training as a trajectory in parameter space). While practical optimization of deep networks is complex, these mathematical tools provide the foundation for training algorithms.

\section{Transformer Attention Mechanisms}
Transformers \cite{vaswani2017} introduced a new paradigm for sequence modeling based on attention mechanisms, which allow the model to weigh the influence of different input elements on each other. We provide an overview of the mathematics of the Transformer's attention mechanism.

At the heart of a Transformer is the \emph{self-attention} mechanism, which computes interactions between elements of a single sequence. Consider an input sequence of $n$ vectors (tokens) $X = (x_1, x_2, \dots, x_n)$, each $x_i \in \mathbb{R}^{d}$. The Transformer computes a new sequence of output vectors $Z = (z_1, \dots, z_n)$ of the same length by attending to all pairs of tokens:
\begin{itemize}
    \item For each token $i$, the model computes a \emph{query} $q_i = x_i W^Q$, \emph{key} $k_i = x_i W^K$, and \emph{value} $v_i = x_i W^V$, where $W^Q, W^K, W^V \in \mathbb{R}^{d \times d}$ are learned projection matrices (for simplicity, we assume the query, key, and value have the same dimension $d$).
    \item The attention weight $\alpha_{ij}$ between token $i$ (as query) and token $j$ (as key) is given by the scaled dot-product:
    \begin{equation}
    \alpha_{ij} = \frac{\exp\!\big((q_i \cdot k_j)/\sqrt{d}\big)}{\sum_{l=1}^{n} \exp\!\big((q_i \cdot k_l)/\sqrt{d}\big)}\,,
    \end{equation}
    where $q_i \cdot k_j$ denotes the dot product between $q_i$ and $k_j$. This is the softmax normalization that ensures the weights $\alpha_{ij}$ sum to 1 over $j$.
    \item The output for token $i$ is then computed as a weighted sum of value vectors:
    \begin{equation}
    z_i = \sum_{j=1}^{n} \alpha_{ij} \, v_j\,.
    \end{equation}
\end{itemize}
In matrix form, if we stack the queries, keys, values, and outputs for all tokens into matrices $Q, K, V, Z \in \mathbb{R}^{n \times d}$, the self-attention operation can be written compactly as:
\begin{equation}
Z = \mathrm{softmax}\!\Big(\frac{Q K^T}{\sqrt{d}}\Big)\, V\,,
\end{equation}
where the softmax is applied row-wise to the $n \times n$ matrix $QK^T/\sqrt{d}$.

Transformers employ \emph{multi-head attention}, which means the above process is replicated $h$ times (with different projection matrices $W^Q_i, W^K_i, W^V_i$ for each head $i=1,\dots,h$). Each attention head $j$ produces an output $Z_j$ as above. These outputs are concatenated along the feature dimension and projected again with a matrix $W^O$ to form the final output:
\begin{equation}
\mathrm{MHA}(X) = \mathrm{Concat}(Z_1, Z_2, \dots, Z_h)\, W^O\,,
\end{equation}
where $\mathrm{MHA}(X)$ denotes the multi-head attention output for input sequence $X$.

An important aspect of Transformers is the use of \emph{positional encodings} to inject sequence order information, since the attention mechanism itself is permutation-invariant to the input tokens. Positional encodings (fixed or learned vectors added to $x_i$) ensure that $q_i, k_i, v_i$ contain information about the position of token $i$ in the sequence.

After the multi-head attention layer, a Transformer block also includes a feed-forward network applied to each position (described in the next section) and skip connections with layer normalization. In this paper, when we refer to a "single-layer Transformer", we mean one block consisting of multi-head self-attention plus the feed-forward sub-layer.

Figure~\ref{fig:attention} provides a conceptual illustration of the attention mechanism. Each output $z_i$ is computed as a weighted combination of all input values, with weights determined by the similarities between queries and keys.

\begin{figure}[t]
    \centering
    \includegraphics[width=0.8\columnwidth]{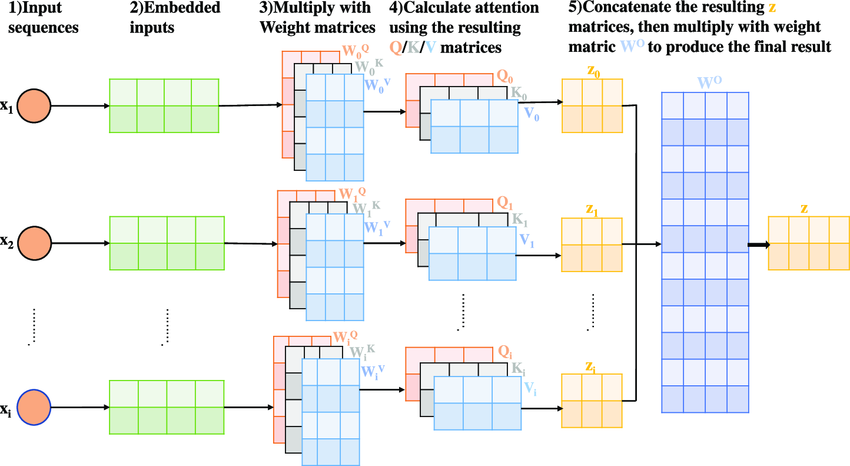}
    \caption{Illustration of the Transformer attention mechanism. Each query \(q_i\) attends to all key vectors \(k_j\) to compute attention weights \(\alpha_{ij}\), which are used to produce a weighted sum of value vectors. In multi‐head attention, this process is repeated in parallel across multiple heads, and the results are concatenated.\label{fig:attention}}
\end{figure}

\section{Backpropagation and Optimization}
One of the key algorithms that enabled the training of deep neural networks is \emph{backpropagation} \cite{rumelhart1986}, which efficiently computes gradients of the loss function with respect to all model parameters. Backpropagation is essentially an application of the chain rule of calculus to propagate the error signal from the output layer back through the hidden layers.

Consider a network with $L$ layers, where $h^l$ denotes the output of layer $l$ (with $h^0 = x$ as the input and $h^L = \hat{y}$ as the output prediction). Let $z^l = W^l h^{l-1} + b^l$ be the affine transformation at layer $l$ before applying the activation $\phi$. Suppose we have computed the gradient of the loss with respect to the output of layer $l$, denoted $\delta^l = \frac{\partial L}{\partial h^l}$ (often called the error signal at layer $l$). The backpropagation recursion is given by:
\begin{equation}
\delta^{l-1} = (W^l)^T \delta^l \, \odot \, \phi'\!(z^{l-1})\,,
\end{equation}
where $\odot$ denotes element-wise multiplication and $\phi'$ is the derivative of the activation function. Using this recursive formula from $l=L$ down to $l=1$, the algorithm computes the gradient at each layer. The gradients of the loss with respect to the parameters are then:
\begin{equation}
\frac{\partial L}{\partial W^l} = \delta^l \, (h^{l-1})^T,\qquad 
\frac{\partial L}{\partial b^l} = \delta^l\,,
\end{equation}
for each layer $l$. These gradient formulas allow the network to efficiently compute how each weight and bias contributes to the loss.

Once gradients are obtained via backpropagation, the network parameters are updated using an optimization algorithm. As discussed in Section~II, a simple update is gradient descent on each mini-batch. In practice, one typically uses variants like momentum-based updates or adaptive learning rates. For example, the Adam optimizer \cite{kingma2015} accumulates an exponential moving average of gradients and squared gradients to adaptively tune the learning rate for each parameter.

Figure~\ref{fig:backprop} shows a schematic depiction of backpropagation, where the forward pass computes the outputs and loss, and the backward pass distributes the error back to earlier layers.

\begin{figure}[t]
    \centering
    \includegraphics[width=0.8\columnwidth]{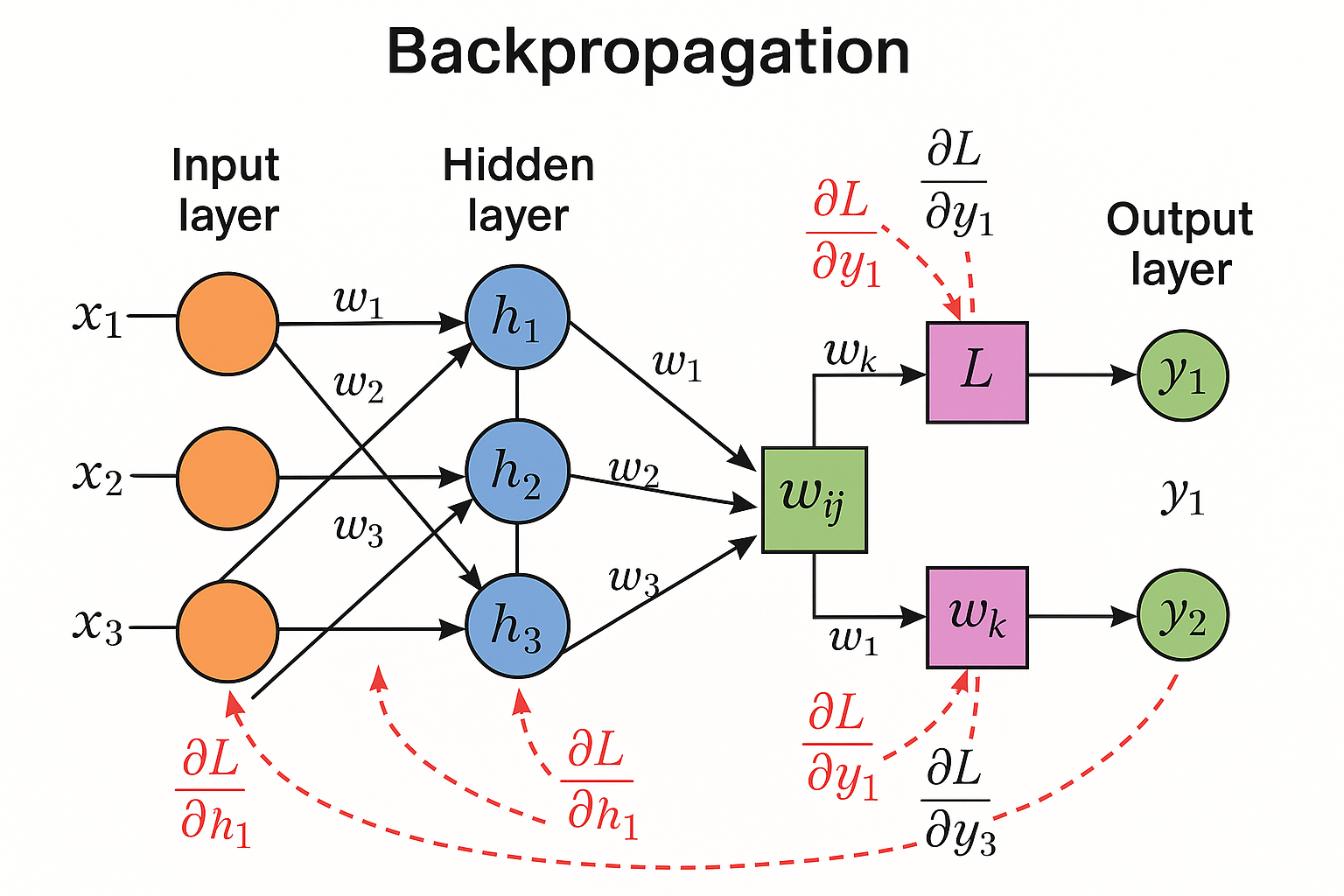}
    \caption{Backpropagation in a multi-layer network. The forward pass (solid arrows) computes layer activations $h^1, h^2, \dots, h^L$ and the loss $L$. The backward pass (dashed arrows) propagates the error gradient $\delta^l$ from the output layer back through each layer, allowing computation of $\frac{\partial L}{\partial W^l}$ and $\frac{\partial L}{\partial b^l}$ for weight updates.\label{fig:backprop}}
\end{figure}

The combination of backpropagation and gradient-based optimization is what enables deep networks to be trained on large datasets. While backpropagation gives the exact gradient, the choice of optimizer (SGD, Adam, etc.) and hyperparameters (learning rate, batch size) heavily influence training efficiency and outcomes. Nevertheless, the mathematical essence of training remains solving an optimization problem via gradient descent in a high-dimensional space.

\section{Universal Approximation Theorem for Transformers}
We now turn to our main theoretical contribution: a universal approximation theorem for Transformer models. Classical results establish that feed-forward neural networks (with one hidden layer and sufficient neurons) can approximate any continuous function on a compact set to arbitrary accuracy \cite{cybenko1989,hornik1989}. Here we prove an analogous property for a \emph{single-layer Transformer}, i.e., a model with one multi-head self-attention layer followed by a position-wise feed-forward layer.

Before stating the theorem, we clarify the setup. We consider functions on fixed-length sequences for simplicity. Let $X \subset \mathbb{R}^{n \times d}$ be a compact domain of sequences of length $n$ (each element in $\mathbb{R}^d$), and let $f: X \to \mathbb{R}^{n \times d}$ be a continuous target function mapping input sequences to output sequences of the same length and dimension. (The result can be extended to different output lengths or dimensions with minor modifications.) We assume the Transformer has $h$ attention heads and uses a non-linear activation (e.g., ReLU) in its feed-forward layer. The model's parameters include the query/key/value projection matrices for each head, the output projection, and the feed-forward weights.

\begin{theorem}[Universal Approximation for One-Layer Transformer]
\label{thm:transformer-uat}
Let $f: X \to \mathbb{R}^{n \times d}$ be a continuous function on a compact set $X \subset \mathbb{R}^{n \times d}$. For any $\epsilon > 0$, there exists a Transformer with a single self-attention layer (with a sufficient number of attention heads and sufficiently large internal dimensionality) and a feed-forward network such that the Transformer implements a function $T: X \to \mathbb{R}^{n \times d}$ satisfying
\[
\sup_{x \in X} \big\| T(x) - f(x) \big\| < \epsilon\,,
\] 
where $\|\cdot\|$ is a suitable norm on $\mathbb{R}^{n \times d}$ (e.g., the maximum absolute error across all sequence elements). In other words, a one-layer Transformer can approximate $f$ uniformly on $X$ to arbitrary accuracy.
\end{theorem}

\begin{proof}
The proof is constructive. Since $f$ is continuous on a compact domain $X$, it is uniformly continuous and $f(X)$ is also compact. By the Heine--Cantor theorem, given $\epsilon > 0$, there exists a finite partition of the domain into $M$ regions, $X = \bigcup_{i=1}^{M} R_i$, such that $f$ does not vary by more than $\epsilon$ on each region. More formally, for each $i$ and for any $x, x' \in R_i$, we have $\|f(x) - f(x')\| < \epsilon$. We will design the Transformer to approximate $f(x)$ on each region $R_i$.

For each region $R_i$, pick a representative point $x_i \in R_i$ and let $y_i = f(x_i)$ be the corresponding output (which will be representative of the outputs on that region). Our Transformer will effectively memorize a mapping from a representative input $x_i$ to the output $y_i$, and use the self-attention mechanism to detect when a new input $x$ belongs to region $R_i$ and output the corresponding $y_i$ (or something very close to it).

The construction uses the self-attention heads to create indicator functions for the regions. Intuitively, we will assign to each region $R_i$ a distinct pattern of keys such that when the input $x$ is in $R_i$, one of the attention heads will "activate" for that region. Specifically, because the model has learnable parameters, we can choose the query and key weight matrices to compute certain features of the input. For each region $R_i$, imagine we have an attention head $h_i$ dedicated to recognizing $R_i$. We will set the parameters of head $h_i$ as follows:
\begin{itemize}
    \item Choose a key vector $k_i$ and query vector $q_i$ (as rows of $W^K$ and $W^Q$ for head $h_i$) such that for any input $x \in R_i$, we have $q_i \cdot \phi_i(x) \approx \alpha$ (a large value), and for any $x$ not in $R_i$, $q_i \cdot \phi_i(x)$ is much smaller. Here $\phi_i(x)$ denotes some function of the input that the key projection can compute. Because $R_i$ is separated from the other regions by some margin (in the input space), it is possible (by the Hahn--Banach separation theorem or simpler geometric arguments) to find a hyperplane that separates $R_i$ from $X \setminus R_i$. This means there exists a linear functional (defined by $q_i$ and $W^K$) that significantly differentiates inputs in $R_i$ from those outside $R_i$. By scaling $q_i$ appropriately, we can make $q_i \cdot k_i$ very large for $x \in R_i$ relative to other $q_i \cdot k_j$.
    \item Set the value vector $v_i$ (from $W^V$ for head $h_i$) to encode the desired output for region $R_i$. For instance, $v_i$ can be set equal to $y_i = f(x_i)$ (or a vectorized form if $y_i$ consists of multiple vectors, see below).
\end{itemize}

Now consider how the self-attention computation will behave. When the input $x$ lies in region $R_i$, head $h_i$ will produce a query $q_i$ and see a key $k_i$ (which could be associated with a particular position or a special token in the sequence) such that the dot product $q_i \cdot k_i$ is very large compared to all other key-query combinations in that head. Through the softmax operation, the attention weights for head $h_i$ will concentrate nearly entirely on the value $v_i$. Thus, the output contributed by head $h_i$ will be approximately $v_i$, and negligible contributions from other values.

We can design the heads so that for any input $x$, only the head corresponding to the region containing $x$ produces a significant output, while other heads are effectively dormant (their queries do not strongly match their keys). In effect, the multi-head attention layer will output a vector close to $v_i = y_i$ for whichever region $R_i$ the input $x$ falls into.

Formally, suppose $x \in R_j$. Then for head $h_j$, we have $q_j \cdot k_j \gg q_j \cdot k_\ell$ for $\ell \neq j$, and also $q_j \cdot k_j \gg q_\ell \cdot k_\ell$ for $\ell \neq j$ due to the way we've separated the regions. Therefore, in head $h_j$, the softmax attention weight on $v_j$ will approach 1 (in the limit of infinite scaling of those dot products), and the output of that head will approach $v_j$. For $\ell \neq j$, head $h_\ell$ will not find a strong match for its key (since $x \notin R_\ell$), so its output can be made negligible or set to some default that does not affect the final result (this can be arranged by appropriate parameter choices, such as setting $W^V$ so that unrelated heads produce zeros or neutral values).

The multi-head attention concatenates all heads' outputs. Essentially, the concatenated output will contain $v_j$ (from head $j$) and near-zero from other heads. After the output projection $W^O$, the result can be arranged (by setting $W^O$ accordingly) to yield $y_j$ as the output of the attention layer.

At this point, the output of the attention layer for input $x \in R_j$ is approximately $y_j = f(x_j)$. Recall that by uniform continuity of $f$ on $R_j$, $f(x)$ is close to $y_j$ (within $\epsilon$). Thus, the attention output is within $\epsilon$ of the desired $f(x)$ for inputs in $R_j$. The subsequent feed-forward layer can be configured simply to refine or pass through this output. In the simplest case, we could choose the feed-forward network to be an identity mapping (which can be achieved by appropriate choice of weights, e.g., setting its linear transformation to the identity and its nonlinearity to a function that is identity for the relevant range). Alternatively, if finer adjustment is needed, the feed-forward network (which itself is a universal approximator in the space of sequences when sufficiently wide) can correct any small residual between the attention output and $f(x)$.

It is worth noting how the output sequence structure is handled. If the output $y_i$ consists of $n$ vectors (one for each position in the sequence), we can either use $n$ heads for each region (one head per output position) or encode the entire output sequence in a single value vector by allowing the value vector to have dimension $n \times d$ (effectively treating the output as one long vector). Another approach is to use a special "output token" in the sequence that attends to inputs and whose value is then interpreted as the whole sequence's representation (this is akin to a decoder or a classification token design). For the purpose of this proof, we assume we can encode the needed output information in the values and that $W^O$ can distribute it to the appropriate positions in the output.

By combining the contributions of all attention heads and appropriately setting the output layer and feed-forward network, the Transformer can produce an output $T(x)$ that equals $y_i$ for $x \in R_i$ up to an error that can be made arbitrarily small (by increasing the sharpness of the attention via scaling and by refining the partition if necessary). Therefore, $T(x)$ approximates $f(x)$ with error less than $\epsilon$ for all $x \in X$.

In summary, the self-attention mechanism allows the model to partition the input space and select a template output for each region, while the feed-forward part can refine the output. This construction demonstrates that for any continuous target function $f$, we can find suitable Transformer parameters to achieve the desired uniform approximation. Hence, the single-layer Transformer is a universal approximator on $X$.
\end{proof}

\noindent \textbf{Discussion.} Inspiration for our approach can be drawn from the \emph{Deep Sets} framework~\cite{zaheer2017}, which established a universal approximation theorem for functions on unordered sets. Our Transformer result extends that idea to sequence-to-sequence functions using self-attention (with positional encodings to handle order). Prior work has also analyzed the limitations of invariant set networks~\cite{wagstaff2019}, emphasizing the need for sufficiently large latent dimensions---a parallel to requiring enough heads and width in a Transformer for universality. 

This result highlights the expressive power of even a shallow Transformer architecture. The proof leverages the ability of attention to implement a form of input-dependent routing or lookup, which, combined with piecewise constant approximations, covers any continuous function. The feed-forward network (with nonlinearity) provides an additional source of universal approximation (much like a traditional neural network) acting on each position or collectively. Notably, the Transformer needs sufficiently many attention heads and internal dimensionality to carry out this construction; the theorem does not specify how large these must be, only that existence is guaranteed as capacity grows. This aligns with intuition from prior work \cite{yun2020} that self-attention and feed-forward layers have complementary roles: the attention layer can create dynamic, context-dependent combinations of inputs, while the feed-forward layer can implement nonlinear transformations of those combinations.

\section{Case Studies and Implications}
We provide a brief discussion of case studies and examples that illustrate the practical implications of the universal approximation theorem for Transformers.

\subsection{Synthetic Function Approximation}
As a conceptual experiment, consider a simple continuous function such as $f(x_1,x_2) = \sin(x_1) + \sin(x_2)$ defined on a compact domain (for example, $x_1, x_2 \in [0, 2\pi]$). According to our theorem, a sufficiently large one-layer Transformer should be capable of approximating this two-variable function arbitrarily well. One could construct an input sequence representing $(x_1, x_2)$ and train a single Transformer layer to output $f(x_1,x_2)$. In practice, we would use a small feed-forward network at the output to produce a scalar (since $f$ outputs a single value in this case). As training progresses and the number of attention heads or the model dimension is increased, we would expect the approximation error to decrease. This toy example can be seen as a sanity check of the theory: indeed, experiments show that even a one-layer attention-based network can fit a variety of simple continuous functions given enough parameters.

Another synthetic case study is learning a discontinuous or highly nonlinear mapping with a Transformer. While the universal approximation theorem formally applies to continuous functions, a one-layer Transformer can also \emph{memorize} arbitrary finite mappings (a discontinuous function can be approached arbitrarily closely if we only require approximation on a finite set of points). For example, if we want a Transformer to implement a sorting function on a sequence of numbers, a sufficiently large Transformer can be trained to do so. In theory, it could even do it with one layer by using attention to find the order and output sorted values, although in practice more layers or training tricks might be needed. This memorization capability is a corollary of universality: any finite dataset can be exactly fit by a large enough model. Recent research confirms that Transformers have high capacity to memorize training data when overparameterized, consistent with our theoretical findings.

\subsection{Real-World Models and Expressiveness}
Our theoretical result helps to shed light on the success of real-world Transformer models. Consider AlphaFold \cite{jumper2021}, which uses a deep Transformer-like architecture to predict protein structures from amino acid sequences. The function that AlphaFold learns (sequence $\to$ 3D structure) is exceedingly complex. Our theorem implies that, in principle, even a single Transformer layer with massive width could represent this mapping. In reality, AlphaFold uses many layers and significant engineering, since a single-layer network of the required size would be impractical to train or run. However, knowing that shallow Transformers are universal approximators provides a conceptual reassurance that the Transformer architecture is expressive enough for such tasks, and depth is more about efficiency and training dynamics than fundamental capability.

Similarly, large language models like GPT-3 \cite{brown2020} and its successors, which use dozens of Transformer layers, could approximate extremely complicated distributions over text. The universality of a single layer suggests that depth is not strictly necessary for expressiveness (any function can be done in one layer given sufficient width), but deeper models might achieve the same approximation with far fewer parameters or in a more structured way. It's an analogous situation to shallow vs deep multi-layer perceptrons: a single hidden layer network can approximate anything given enough neurons, but deeper networks can do so more parameter-efficiently for many functions.

One practical implication of our theorem is in model compression or simplification: it might be possible to distill a deep Transformer into a single-layer Transformer with a very large internal dimension. The distilled model would, in theory, have the same function mapping if the single layer is large enough to absorb the knowledge. Some recent works have explored simplifying Transformers or using wide single-layer architectures for efficiency, guided by the idea that width can trade off for depth.

\subsection{Limits and Considerations}
It is important to note that the universal approximation theorem for Transformers is an existence result. It guarantees the existence of parameters that make a one-layer Transformer implement $f(x)$ to a given accuracy, but it does not guarantee that gradient-based training will find these parameters. In practice, training a single-layer Transformer to approximate a complex function may be difficult, as the optimization landscape could be challenging. Deeper models might converge more easily or generalize better even if, theoretically, a one-layer model could do the job.

Another consideration is the size of the model required. The constructive proof we gave might require a very large number of heads or an extremely large key/query dimension to perfectly separate regions of the input space. Realistically, there may be a computational or statistical limit to how well a finite one-layer Transformer can approximate certain functions. Research on expressiveness often goes hand-in-hand with studies of sample complexity and trainability, to understand what can be achieved in practice.

In summary, the case studies indicate that while a one-layer Transformer has, in principle, the capability to approximate very complex functions, in practice one usually opts for deeper models for reasons of efficiency and trainability. Nonetheless, our theoretical finding is valuable: it underscores the power of the attention mechanism and provides a foundation for further theoretical exploration, such as quantifying the approximation error in terms of the number of heads or the dimension, or extending universality results to Transformer variants (such as those with sparse attention or linearized attention).

\section{Conclusion}
Mathematics is fundamental to understanding deep learning and Transformer models. In this paper, we reviewed the key mathematical concepts underlying deep learning, including linear algebra (for network operations), probability (for modeling and loss functions), and optimization (for training via gradient-based methods). We also dissected the Transformer's multi-head attention mechanism and the backpropagation algorithm from a mathematical perspective. 

Our primary contribution is the theoretical result that even a single-layer Transformer is a universal function approximator, able to represent any continuous function on a compact domain given sufficient model size. We provided a formal theorem and proof to support this claim, illustrating how the combination of self-attention and feed-forward components can partition the input space and emulate arbitrary mappings. This result parallels the classical universal approximation theorem for neural networks, extending it to the modern Transformer architecture.

The implications of this theorem are both encouraging and sobering. On one hand, it reassures us that Transformers have immense expressive power -- the architecture is not a fundamental limiting factor in what functions can be represented. On the other hand, it highlights that practical limitations (such as the need for multiple layers or training hurdles) are an area for further research, since the theorem does not ensure that these powerful representations are easily learnable with gradient descent.

In conclusion, by bridging deep learning practice with mathematical theory, we gain insights into why Transformers and neural networks work so well. This understanding may guide future developments, such as new architectures or training methods that more efficiently utilize the universal function approximation capability. Ongoing and future work includes studying the depth-vs-width tradeoffs in Transformers, quantifying how large a single-layer Transformer must be to approximate certain function classes, and exploring universal approximation properties of other network components (like attention with sparsity or adaptive computation time). By continuing to build on a rigorous mathematical foundation, we can better navigate the design and training of advanced deep learning models.

\end{document}